\documentclass{article}

\usepackage[letterpaper,margin=1in]{geometry}
\usepackage{times}
\usepackage[utf8]{inputenc}
\usepackage[T1]{fontenc}
\usepackage{amsmath,amssymb,amsfonts,amsthm,bm}
\usepackage{booktabs,multirow,array}
\usepackage[round]{natbib}
\usepackage{hyperref}
\hypersetup{colorlinks=true,linkcolor=blue,citecolor=blue,urlcolor=blue}
\usepackage{tikz}
\usetikzlibrary{positioning,arrows.meta,calc,fit}

\definecolor{ink}{HTML}{1C2230}
\definecolor{mutedgray}{HTML}{5B6472}
\definecolor{paperteal}{HTML}{0F6B5F}
\definecolor{paperviolet}{HTML}{6B3FA0}
\definecolor{paperred}{HTML}{C0392B}
\definecolor{lightline}{HTML}{D8D2C4}
\definecolor{tealbg}{HTML}{EEF6F4}
\definecolor{violetbg}{HTML}{F0E9F8}
\definecolor{warmbg}{HTML}{F4F2EC}
\definecolor{sand}{HTML}{E9E5DB}
\definecolor{sanddark}{HTML}{B3ACA0}

\newcommand{\CGA}{\textup{\textsc{H-CGA}}}
\newcommand{\E}{\mathbb{E}}
\newcommand{\Prob}{\mathbb{P}}
\newcommand{\Normal}{\mathcal{N}}
\newcommand{\Uniform}{\mathcal{U}}

\newcommand{\TopK}{\operatorname{TopK}}
\newcommand{\ind}{\mathbf{1}}
\newcommand{\stopgrad}{\operatorname{stopgrad}}
\newcommand{\R}{\mathbb{R}}

\newtheorem{theorem}{Theorem}
\newtheorem{corollary}{Corollary}
\newtheorem{proposition}{Proposition}

\title{Hierarchical Copula-Gumbel-Top-\texorpdfstring{$K$}{K} Routing:\\
Two-Sided Dependence Control for Frozen Mixture-of-Experts\\
at Fixed Per-Token Routing Laws}

\author{Richard Yi Da Xu\\
Hong Kong Baptist University and TadReamk Limited\\
\texttt{xuyida@hkbu.edu.hk}, \texttt{richard@tadreamk.com}}
\date{}

\begin{document}
\maketitle

\begin{abstract}
A stochastic Gumbel-Top-$K$ router defines, for every token of a
mixture-of-experts (MoE) model, a \emph{routing law}: a distribution over
ordered expert lists and mixture weights.  We ask which \emph{joint}
distributions over the routing choices of different tokens are reachable
while every individual token's complete routing law is held exactly fixed.
We give a two-sided construction, \emph{Hierarchical Copula-Gumbel-Top-$K$}
(\CGA{}).  Within a group of related tokens, an exchangeable Gaussian copula
positively correlates the Gumbel perturbations at each expert coordinate,
which can increase within-group expert-set coherence.  Across disjoint pairs of
groups, a tunable antithetic construction introduces a selectable amount of
negative dependence.  We prove
that both operations leave each token's ordered Top-$K$ sample, mixture
weights, and inclusion probabilities identical in distribution to independent
routing \emph{at a routing layer conditioned on its pre-routing logits};
conditional expected expert traffic is preserved as a consequence.  We characterize the resulting
trade-off: positive within-group coupling can only inflate the variance of
realized expert loads relative to independent routing, while nonnegative
cross-group opposition can only reduce it relative to flat coupling at the
same within-group strength.  Coherence and load dispersion are thus controlled
by two complementary dependence dials on the invariance constraint surface.
Because the base model is
untouched, the dials can be driven by a small controller over frozen features,
trainable with a score-function estimator: the frozen network is evaluated
only in the forward direction, and gradients are confined to the controller.
An initial small-scale pilot validates the mechanism and the training route,
but does not establish task-level fine-tuning gains.
\end{abstract}

%==============================================================================
\section{Introduction}
%==============================================================================
Sparse MoE layers scale language models by evaluating only a few experts for
each token \citep{shazeer2017outrageously,lepikhin2021gshard,fedus2022switch}.
Modern MoEs commonly use top-$K$ routing: a router scores all experts, sends a
token to its $K$ highest-scoring experts, and combines their outputs.  Under
the standard stochastic formulation, Gumbel-Top-$K$ sampling
\citep{kool2019gumbel}, the router assigns each token a routing law---a
Plackett--Luce distribution over ordered expert lists.  Nearly all work on MoE
routing modifies this per-token law: fine-tuning changes the logits, auxiliary
losses reshape the gates, and similarity-aware routers alter individual
selections \citep{nguyen2025graphtokens,omi2025simbal}.

We study a different, largely unexamined degree of freedom.  Under
independent stochastic routing, the joint routing distribution of a batch is
a \emph{product measure} over tokens---a default nobody chose
(Figure~\ref{fig:scatter}, top).  Holding every
token's routing law \emph{exactly} fixed, the joint distribution over the
choices of different tokens is still free: Sklar's theorem separates marginals
from dependence \citep{sklar1959fonctions,nelsen2006copulas}, and the routing
choices of a frozen MoE are a collection of discrete marginals awaiting a
dependence structure.  Figure~\ref{fig:marginals-joint} makes the degree of
freedom concrete: two joint laws with identical per-token marginals can
differ arbitrarily in how the tokens' choices co-move, and every dial
practitioners tune today---auxiliary balance losses
\citep{shazeer2017outrageously}, logit biasing, capacity tuning---moves the
marginals, not the dependence.  This paper asks: \emph{which joint routing
behaviors are reachable on this invariance constraint surface, and what do
the reachable extremes trade off?}

\begin{figure}[t]
\centering
\resizebox{\textwidth}{!}{%
\begin{tikzpicture}[font=\small,
  tok/.style={draw=sanddark,fill=white,rounded corners=2pt,minimum width=1.5cm,minimum height=0.5cm},
  exoff/.style={draw=lightline,fill=white,rounded corners=2pt,minimum width=1.05cm,minimum height=0.5cm,text=sanddark},
  lk/.style={opacity=0.75}]
% ---------- panel (a): independent ----------
\begin{scope}
  \node[font=\small\bfseries,paperred] at (3.6,3.15) {(a) independent routing (today's default)};
  \node[tok] (a1) at (0,2.4) {\texttt{for}};
  \node[tok] (a2) at (0,1.7) {\texttt{(int}};
  \node[tok] (a3) at (0,1.0) {\texttt{i=0;}};
  \node[tok] (a4) at (0,0.3) {\texttt{i<n;}};
  \node[draw=sanddark,dashed,rounded corners=4pt,fit=(a1)(a4),inner sep=5pt] (agrp) {};
  \node[mutedgray,font=\scriptsize,anchor=north] at (agrp.south) {one code phrase (a ``group'')};
  \foreach \i/\y in {1/2.75,2/2.25,3/1.75,4/1.25,5/0.75,6/0.25}{
    \node[exoff,draw=paperred,fill=paperred!8,text=paperred] (ae\i) at (6.6,\y) {E\i};}
  \draw[lk,paperred] (a1.east) -- (ae1.west);
  \draw[lk,paperred] (a1.east) -- (ae4.west);
  \draw[lk,paperred] (a2.east) -- (ae3.west);
  \draw[lk,paperred] (a2.east) -- (ae6.west);
  \draw[lk,paperred] (a3.east) -- (ae2.west);
  \draw[lk,paperred] (a3.east) -- (ae5.west);
  \draw[lk,paperred] (a4.east) -- (ae1.west);
  \draw[lk,paperred] (a4.east) -- (ae6.west);
  \node[paperred,font=\scriptsize,align=center] at (6.6,-0.95)
    {6 distinct experts touched:\\ no reuse across the phrase};
\end{scope}
% ---------- panel (b): coupled ----------
\begin{scope}[xshift=8.7cm]
  \node[font=\small\bfseries,paperteal] at (3.6,3.15) {(b) copula-coupled routing (this paper)};
  \node[tok] (b1) at (0,2.4) {\texttt{for}};
  \node[tok] (b2) at (0,1.7) {\texttt{(int}};
  \node[tok] (b3) at (0,1.0) {\texttt{i=0;}};
  \node[tok] (b4) at (0,0.3) {\texttt{i<n;}};
  \node[draw=paperteal,dashed,rounded corners=4pt,fit=(b1)(b4),inner sep=5pt] (bgrp) {};
  \node[mutedgray,font=\scriptsize,anchor=north] at (bgrp.south) {same group, same marginals};
  \foreach \i/\y in {1/2.75,3/1.75,6/0.25}{
    \node[exoff] (be\i) at (6.6,\y) {E\i};}
  \node[draw=paperteal,fill=tealbg,rounded corners=2pt,minimum width=1.05cm,minimum height=0.5cm,text=paperteal,thick] (bh2) at (6.6,2.25) {E2};
  \node[draw=paperteal,fill=tealbg,rounded corners=2pt,minimum width=1.05cm,minimum height=0.5cm,text=paperteal,thick] (bh4) at (6.6,1.25) {E4};
  \node[draw=paperteal,fill=tealbg,rounded corners=2pt,minimum width=1.05cm,minimum height=0.5cm,text=paperteal,thick] (bh5) at (6.6,0.75) {E5};
  \draw[lk,paperteal] (b1.east) -- (bh2.west);
  \draw[lk,paperteal] (b1.east) -- (bh5.west);
  \draw[lk,paperteal] (b2.east) -- (bh2.west);
  \draw[lk,paperteal] (b2.east) -- (bh5.west);
  \draw[lk,paperteal] (b3.east) -- (bh2.west);
  \draw[lk,paperteal] (b3.east) -- (bh4.west);
  \draw[lk,paperteal] (b4.east) -- (bh2.west);
  \draw[lk,paperteal] (b4.east) -- (bh5.west);
  \node[paperteal,font=\scriptsize,align=center] at (6.6,-0.95)
    {3 distinct experts touched:\\ coherent, cache-friendly};
\end{scope}
\end{tikzpicture}}
\caption{The problem: every token rolls its own dice.  (a)~Under independent
Gumbel-Top-$2$ routing, the tokens of one code phrase scatter across six
distinct experts---the joint routing law is a product measure,
$\Prob(\text{routing})=\prod_t p_t(e_t)$.  (b)~Coupling the routing noise
draws makes related tokens reuse experts,
$\Prob(\text{routing})\neq\prod_t p_t(e_t)$, while each token's own routing
law $p_t$---and hence every expected expert load---is exactly unchanged
(Theorem~\ref{thm:law}).}
\label{fig:scatter}
\end{figure}
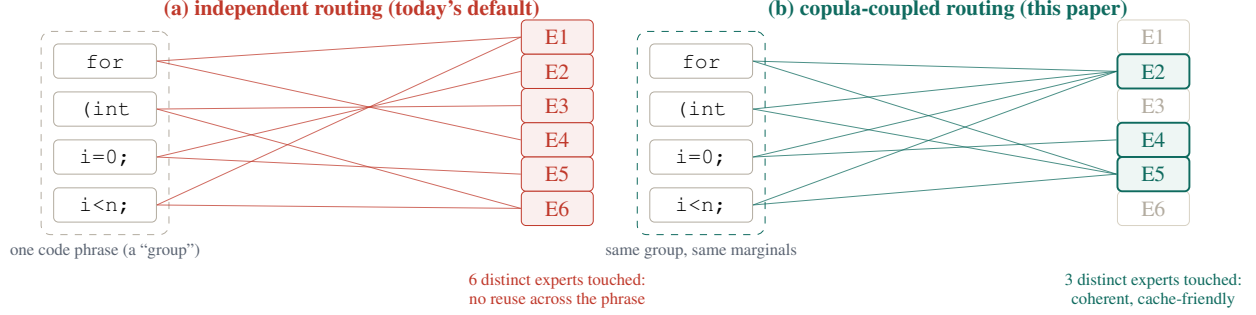

\begin{figure}[t]
\centering
\begin{tikzpicture}[font=\small,x=0.5cm,y=0.5cm]
% ---------- shared marginal p = (0.35,0.25,0.25,0.15) ----------
% panel macro drawn twice with different cell shades
% (a) independent product measure, red
\begin{scope}
  \node[font=\small\bfseries,paperred,anchor=south] at (2,5.4) {independent---today's default};
  % column marginal bars (identical in both panels)
  \foreach \j/\h in {0/1.0,1/0.714,2/0.714,3/0.429}
    \fill[mutedgray,opacity=0.45] (\j+0.18,4.15) rectangle (\j+0.82,4.15+\h);
  % row marginal bars
  \foreach \i/\h in {0/1.0,1/0.714,2/0.714,3/0.429}
    \fill[mutedgray,opacity=0.45] (4.2,4-\i-0.82) rectangle (4.2+\h,4-\i-0.18);
  % joint cells (shade percent of paperred)
  \foreach \i/\row in {0/{65,53,53,40},1/{53,44,44,34},2/{53,44,44,34},3/{40,34,34,27}}
    \foreach [count=\jj from 0] \s in \row
      \fill[paperred!\s] (\jj,3-\i) rectangle (\jj+1,4-\i);
  \draw[white,line width=0.8pt] (1,0) -- (1,4) (2,0) -- (2,4) (3,0) -- (3,4)
    (0,1) -- (4,1) (0,2) -- (4,2) (0,3) -- (4,3);
  \draw[lightline,thick] (0,0) rectangle (4,4);
  \node[mutedgray,font=\scriptsize,anchor=north] at (2,-0.1) {expert of token $t'$};
  \node[mutedgray,font=\scriptsize,rotate=90,anchor=south] at (-0.25,2) {expert of token $t$};
\end{scope}
% middle text
\node[align=center,font=\small] at (10.2,2.6)
  {\textbf{keep the marginals}\\[1pt]
   \textcolor{paperteal}{\textbf{change the dependence}}};
\node[align=center,font=\scriptsize,mutedgray] at (10.2,0.9)
  {existing dials move the bars;\\ a copula re-arranges the cells};
% (b) coupled, teal
\begin{scope}[xshift=7.8cm]
  \node[font=\small\bfseries,paperteal,anchor=south] at (2,5.4) {coupled---what we want};
  \foreach \j/\h in {0/1.0,1/0.714,2/0.714,3/0.429}
    \fill[mutedgray,opacity=0.45] (\j+0.18,4.15) rectangle (\j+0.82,4.15+\h);
  \foreach \i/\h in {0/1.0,1/0.714,2/0.714,3/0.429}
    \fill[mutedgray,opacity=0.45] (4.2,4-\i-0.82) rectangle (4.2+\h,4-\i-0.18);
  \foreach \i/\row in {0/{100,35,35,27},1/{35,78,30,24},2/{35,30,78,24},3/{27,24,24,55}}
    \foreach [count=\jj from 0] \s in \row
      \fill[paperteal!\s] (\jj,3-\i) rectangle (\jj+1,4-\i);
  \draw[white,line width=0.8pt] (1,0) -- (1,4) (2,0) -- (2,4) (3,0) -- (3,4)
    (0,1) -- (4,1) (0,2) -- (4,2) (0,3) -- (4,3);
  \draw[lightline,thick] (0,0) rectangle (4,4);
  \node[mutedgray,font=\scriptsize,anchor=north] at (2,-0.1) {expert of token $t'$};
  \node[mutedgray,font=\scriptsize,rotate=90,anchor=south] at (-0.25,2) {expert of token $t$};
\end{scope}
\end{tikzpicture}
\caption{Same marginals, different joint.  Two joint routing laws for a pair
of tokens over four experts, drawn on one shared color scale; the gray bars
are the row and column sums (each token's marginal routing law) and are
\emph{identical} in both panels.  Left: the independent product measure.
Right: a convex mixture of the product with the diagonal, which concentrates
mass on ``both tokens pick the same expert'' while leaving every row and
column sum unchanged.  Existing routing dials (auxiliary losses, logit
biasing, capacity tuning) move the bars; a copula re-arranges the cells.}
\label{fig:marginals-joint}
\end{figure}
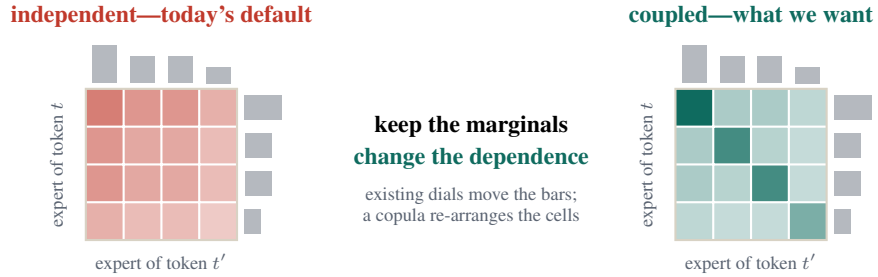

The question has practical stakes in both directions of dependence.
\emph{Positive} cross-token dependence makes the tokens of a phrase, entity,
or code identifier more likely to reuse the same experts---local coherence
and fewer distinct experts per group.  This may improve expert locality,
depending on the execution system.  But bunching grouped tokens onto
shared experts makes their inclusion counts positively correlated, so
realized per-expert loads become burstier even though conditional expected
loads at that layer are provably unchanged.  \emph{Negative} cross-group dependence pushes the other
way: anticorrelating the demand of different groups reduces the variance of
realized loads.  We show both directions are simultaneously available in one
hierarchical construction, and both preserve every token's routing law.

We propose \CGA{}, shown in Figure~\ref{fig:architecture}.  The frozen router
produces logits $\ell_{te}$ for token $t$ and expert $e$; ordinary stochastic
routing adds independent Gumbel noises and takes the top $K$ perturbed
logits.  \CGA{} organizes the noise hierarchically.  Within a group $g$ of
tokens, the noises at each expert coordinate share a group latent
$\zeta_{ge}$ through an exchangeable Gaussian copula: related tokens receive
positively correlated perturbations, expert by expert.  Across a disjoint
pair of groups $(g,g')$, the latents have a tunable antithetic relationship
with opposition strength $\alpha_{g,g'}\in[0,1]$.  At
$\alpha_{g,g'}=1$, whatever random push a pair member gives its tokens toward
expert $e$, its partner receives the opposite push; at $\alpha_{g,g'}=0$, the
two groups are independent.  The noise vector of any single token retains its
original i.i.d.\ Gumbel distribution throughout.

The method leaves each token's marginal selection distribution unchanged and
acts only on the cross-token dependence of the routing noise: positive
coordination raises the tendency of related tokens to make matching random
selections, while negative coordination reduces the tendency of distinct
groups to do so simultaneously. For the stochastic router studied here, this
dependence can be introduced without altering any single token's routing law.

\begin{figure}[t]
\centering
\begin{tikzpicture}[
  node distance=1.0cm and 0.65cm,
  frozen/.style={draw=mutedgray,dashed,fill=warmbg,rounded corners,align=center,minimum height=0.8cm,inner sep=4pt},
  arr/.style={-{Stealth[length=2mm]},thick,mutedgray}
]
\node[frozen] (hidden) {frozen hidden\\states $h_t$};
\node[frozen,right=of hidden] (router) {frozen router\\logits $\ell_t$};
\node[draw=paperteal,thick,fill=tealbg,rounded corners,align=center,minimum height=0.8cm,inner sep=4pt,right=of router] (controller) {trainable controller\\$a_g,\rho_g,\alpha_{g,g'}$};
\node[draw=paperviolet,thick,fill=violetbg,rounded corners,align=center,minimum height=0.8cm,inner sep=4pt,below=of controller] (copula) {hierarchical copula\\Gumbel noises};
\node[frozen,left=of copula] (topk) {frozen\\Gumbel-Top-$K$};
\node[frozen,left=of topk] (experts) {frozen experts\\and weighted sum};
\draw[arr] (hidden)--(router);
\draw[arr,paperteal] (router)--(controller);
\draw[arr] (router.south)--(topk.north);
\draw[arr,paperteal] (controller)--(copula);
\draw[arr,paperviolet] (copula)--(topk);
\draw[arr] (topk)--(experts);
\end{tikzpicture}
\caption{\CGA{} as a routing-side dependence layer.  Only the controller is
trainable.  The base router, experts, and mixture rule are frozen.  Positive
within-group coupling and tunable cross-group opposition are both generated
in the hierarchical copula stage; the pairing is a fixed, pre-sampling design
choice, while the controller may set its opposition strength
$\alpha_{g,g'}$.  At $a_g=0$ the sampler is independent
Gumbel-Top-$K$ routing exactly.}
\label{fig:architecture}
\end{figure}
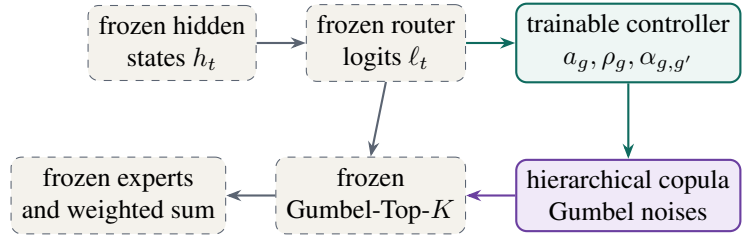

\paragraph{Contributions.}
\begin{enumerate}
  \item We introduce \CGA{}, a hierarchical copula layer over the routing
  noise of a frozen stochastic Gumbel-Top-$K$ MoE that controls cross-token
  dependence in both directions---positive within-group coupling for local
  expert-set coherence and tunable negative cross-group coupling for
  load-variance control---without touching router logits, experts, or any per-token
  routing law.
  \item We prove \emph{full per-token routing-law invariance} for the entire
  hierarchy: the selected ordered Top-$K$ list and gate-based mixture weights
  of every token have the same distribution as under independent
  Gumbel-Top-$K$ routing.  Consequently, conditional expected
  expert-inclusion counts at that layer are preserved
  (Theorem~\ref{thm:law}, Corollaries~\ref{cor:load} and~\ref{cor:hier}).
  \item We characterize the coherence--dispersion trade-off on the invariance
  surface (Proposition~\ref{prop:variance}): relative to independent routing,
  flat positive coupling can only increase the variance of realized expert
  loads; relative to flat coupling at the same strength, every nonnegative
  cross-group opposition strength can only decrease it.  Conditional expected
  loads are identical in all schemes.
  \item We formulate routing-only adaptation as an application: a small
  controller reads frozen features and sets the dependence dials, and is
  trainable with a score-function estimator that evaluates the frozen base
  model only in the forward direction.
\end{enumerate}

%==============================================================================
\section{Background and Problem Setting}\label{sec:background}
%==============================================================================
\paragraph{Top-$K$ MoE routing.}
Let a frozen router map token representation $h_t$ to logits
$\ell_t=(\ell_{t1},\ldots,\ell_{tE})$; Figure~\ref{fig:moe} sketches the
layer.  A deterministic top-$K$ router selects
the $K$ largest entries.  We instead use the standard stochastic
Gumbel-Top-$K$ law: draw $\gamma_{te}\overset{\mathrm{iid}}{\sim}
\mathrm{Gumbel}(0,1)$ and let
\begin{equation}
  (r_{t1},\ldots,r_{tK})
  = \TopK_e\{\ell_{te}+\gamma_{te}\},
  \qquad S_t=\{r_{t1},\ldots,r_{tK}\}.
  \label{eq:gumbel-topk}
\end{equation}

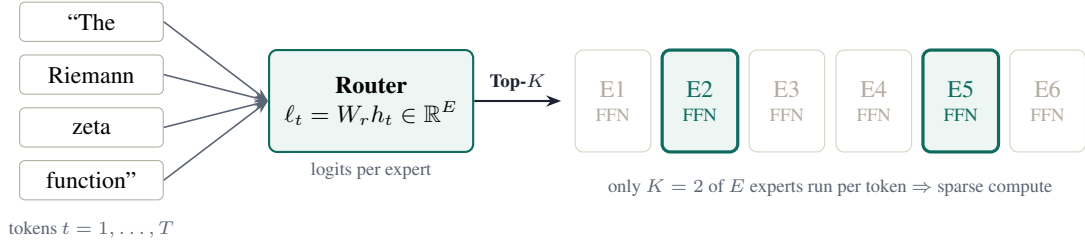
\begin{figure}[t]
\centering
\begin{tikzpicture}[font=\small,
  tok/.style={draw=sanddark,fill=white,rounded corners=2pt,minimum width=1.9cm,minimum height=0.52cm},
  arr/.style={-{Stealth[length=1.8mm]},mutedgray,semithick}]
% tokens
\node[tok] (t1) at (0,2.55) {``The};
\node[tok] (t2) at (0,1.85) {Riemann};
\node[tok] (t3) at (0,1.15) {zeta};
\node[tok] (t4) at (0,0.45) {function''};
\node[mutedgray,font=\scriptsize] at (0,-0.2) {tokens $t=1,\dots,T$};
% router
\node[draw=paperteal,fill=tealbg,thick,rounded corners=3pt,minimum width=2.7cm,minimum height=1.35cm,align=center] (router) at (3.7,1.5)
  {\textbf{Router}\\ $\ell_t = W_r h_t\in\R^{E}$};
\node[mutedgray,font=\scriptsize,anchor=north] at (router.south) {logits per expert};
\foreach \s in {t1,t2,t3,t4} \draw[arr] (\s.east) -- (router.west);
% top-k arrow
\draw[-{Stealth[length=2mm]},ink,thick] (router.east) -- node[above,font=\scriptsize\bfseries]{Top-$K$} ++(1.15,0);
% experts
\foreach \i/\x in {1/6.9,2/8.05,3/9.2,4/10.35,5/11.5,6/12.65}{
  \node[draw=lightline,fill=white,rounded corners=3pt,minimum width=1.0cm,minimum height=1.35cm,align=center,text=sanddark] (e\i) at (\x,1.5) {E\i\\[-2pt]{\scriptsize FFN}};}
\node[draw=paperteal,fill=tealbg,very thick,rounded corners=3pt,minimum width=1.0cm,minimum height=1.35cm,align=center,text=paperteal] at (8.05,1.5) {E2\\[-2pt]{\scriptsize FFN}};
\node[draw=paperteal,fill=tealbg,very thick,rounded corners=3pt,minimum width=1.0cm,minimum height=1.35cm,align=center,text=paperteal] at (11.5,1.5) {E5\\[-2pt]{\scriptsize FFN}};
\node[mutedgray,font=\scriptsize,align=center] at (9.775,0.35)
  {only $K=2$ of $E$ experts run per token $\Rightarrow$ sparse compute};
\end{tikzpicture}
\caption{Mixture-of-experts routing.  A router scores all $E$ experts for
each token; only the top $K$ run, and their outputs are combined by the
gate weights of \eqref{eq:mixture}.  The layer output $y_t$ re-enters the
residual stream.  Sparse conditional compute buys the capacity of $E$
experts at the cost of $K$---provided the load across experts stays
balanced.}
\label{fig:moe}
\end{figure}

The ordered list is an exact Plackett--Luce sample without replacement
\citep{kool2019gumbel}: writing $Z_t=\sum_j\exp(\ell_{tj})$,
\begin{equation}
  \Prob(r_{t1}=e_1,\ldots,r_{tK}=e_K)
  = \prod_{k=1}^{K}
    \frac{\exp(\ell_{te_k})}{Z_t-\sum_{m<k}\exp(\ell_{te_m})},
  \label{eq:plackett-luce}
\end{equation}
i.e.\ softmax sampling without replacement, one slot at a time.
Figure~\ref{fig:gumbel} illustrates the underlying Gumbel-max mechanism:
perturbing each logit with i.i.d.\ Gumbel noise and taking the argmax is an
exact softmax sample, and taking the top $K$ perturbed logits yields
\eqref{eq:plackett-luce}.  One noise draw per expert produces the whole
ordered list; the noise occasionally promotes a lower logit, and that
randomness is the exploration.  We use gate-based mixture weights

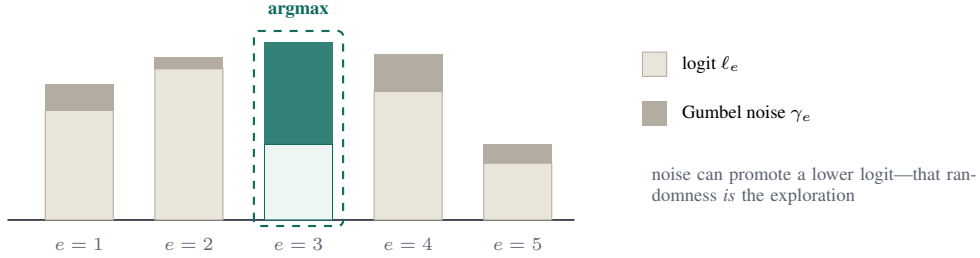
\begin{figure}[t]
\centering
\begin{tikzpicture}[font=\small]
\draw[ink,semithick] (-0.2,0) -- (7.3,0);
% bars: logit (sand) + Gumbel noise (dark) stacks
\foreach \i/\l/\n/\lab in {0/1.45/0.35/1, 1/2.0/0.15/2, 3/1.7/0.5/4, 4/0.75/0.25/5}{
  \fill[sand,draw=sanddark] ({0.3+\i*1.45},0) rectangle ({1.2+\i*1.45},\l);
  \fill[sanddark] ({0.3+\i*1.45},\l) rectangle ({1.2+\i*1.45},{\l+\n});
  \node[mutedgray,font=\scriptsize] at ({0.75+\i*1.45},-0.32) {$e=\lab$};}
% winner: e=3, low logit promoted by a large noise draw
\fill[tealbg,draw=paperteal] (3.2,0) rectangle (4.1,1.0);
\fill[paperteal,opacity=0.85] (3.2,1.0) rectangle (4.1,2.35);
\node[mutedgray,font=\scriptsize] at (3.65,-0.32) {$e=3$};
\draw[paperteal,thick,dashed,rounded corners=2pt] (3.06,-0.08) rectangle (4.24,2.5);
\node[paperteal,font=\scriptsize\bfseries,anchor=south] at (3.65,2.55) {argmax};
% legend
\fill[sand,draw=sanddark] (8.2,1.9) rectangle +(0.32,0.32);
\node[anchor=west,font=\scriptsize] at (8.6,2.06) {logit $\ell_e$};
\fill[sanddark] (8.2,1.25) rectangle +(0.32,0.32);
\node[anchor=west,font=\scriptsize] at (8.6,1.41) {Gumbel noise $\gamma_e$};
\node[anchor=west,font=\scriptsize,mutedgray,align=left,text width=4.4cm] at (8.2,0.45)
  {noise can promote a lower logit---that randomness \emph{is} the exploration};
\end{tikzpicture}
\caption{The Gumbel-max trick.  Each logit is perturbed by an independent
$\mathrm{Gumbel}(0,1)$ draw; the argmax of the perturbed logits is an exact
softmax sample, and the top $K$ of them is an exact Plackett--Luce sample
\eqref{eq:plackett-luce}.  Here a large noise draw promotes expert~3 over
the larger logit of expert~2.}
\label{fig:gumbel}
\end{figure}
\begin{equation}
  w_{te} =
  \frac{\exp(\ell_{te})\,\ind[e\in S_t]}
       {\sum_{j\in S_t}\exp(\ell_{tj})},
  \qquad
  y_t=\sum_{e\in S_t}w_{te}f_e(h_t).
  \label{eq:mixture}
\end{equation}
Other weight rules that are deterministic functions of the complete
Gumbel-noise vector can be used as well.

\paragraph{What ``routing law'' means here.}
Conditioned on a token's frozen logits, its routing law is the joint
distribution of its ordered Top-$K$ list and its weights in
\eqref{eq:mixture}.  It describes what can happen to one token under
stochastic routing.  It does not specify how the random choices of two
different tokens co-vary, nor does it fix realized batch loads.

\paragraph{Scope and level of the guarantee.}
The reference router in this paper is \emph{independent stochastic}
Gumbel-Top-$K$, not deterministic top-$K$.  A deterministic router has a
degenerate routing law, so no nontrivial dependence can be introduced while
preserving that law exactly.  Capacity clipping, token dropping, and
expert-choice allocation occur after the choices in \eqref{eq:gumbel-topk};
they are outside the invariance results below.  All invariance and expected-load
claims are \emph{layer-local}: they condition on the hidden states and logits
entering one routing layer.  If coupling changes the joint hidden-state
distribution at an earlier layer, later-layer logits may change as well.
Therefore the results do not by themselves establish end-to-end invariance of
a multi-layer MoE.

\paragraph{Relation to PEFT.}
Conventional MoE parameter-efficient fine-tuning attaches weight adapters to
experts or learns an additional adapter router
\citep{hu2022lora,liu2026perft}; both allocate trainable
\emph{representational} capacity and change what individual tokens prefer.
The framework studied here exposes a different, orthogonal budget: with all
base parameters $\Theta_{\mathrm{base}}$ frozen---embeddings, attention,
router, and experts---the only trainable object is a controller $\phi$ over
the \emph{dependence} of routing randomness, ranging from one scalar per MoE
layer to a small MLP.  Routing-only adaptation
(Section~\ref{sec:adaptation}) is thus an application of the dependence
framework, not its definition.  It is not a conventional weight-adapter PEFT
method: by construction, it cannot change any token's marginal preference for
an expert or its expected inclusion probability.

%==============================================================================
\section{Hierarchical Copula-Gumbel-Top-\texorpdfstring{$K$}{K}}\label{sec:method}
%==============================================================================
This section introduces cross-token coordination into a stochastic router
while leaving each token's marginal routing law unchanged. Throughout, the
marginal routing law denotes the distribution of a single token's ordered
expert list under repeated routing, and the joint routing law denotes the
dependence structure among the selections of distinct tokens routed together;
the frozen router fixes the former, and \CGA{} acts only on the latter.

The construction rests on the classical separation of a joint law into
marginals and dependence (Figure~\ref{fig:sklar}): by Sklar's theorem
\citep{sklar1959fonctions}, every joint distribution factors as
$F(x_1,\ldots,x_n)=C(F_1(x_1),\ldots,F_n(x_n))$ with $C$ a copula on
$[0,1]^n$, uniquely when the marginals are continuous.  Independence is
exactly the copula density $c\equiv1$; changing $C$ never touches the
$F_i$.  \CGA{} instantiates this with a Gaussian copula on the routing
noise.

\begin{figure}[t]
\centering
\begin{tikzpicture}[font=\small,
  arr/.style={-{Stealth[length=1.8mm]},mutedgray,semithick}]
\node[draw=ink,thick,rounded corners=3pt,minimum width=3.5cm,minimum height=1.5cm,align=center] (joint) at (0,0)
  {Joint distribution\\[2pt] $F(x_1,\dots,x_n)$};
\node[font=\Large,ink] at (2.25,0) {$=$};
\node[draw=paperteal,thick,fill=tealbg,rounded corners=3pt,minimum width=3.3cm,minimum height=1.5cm,align=center] (marg) at (4.55,0)
  {Marginals\\[2pt] $F_1,\ \dots,\ F_n$};
\node[font=\Large,ink] at (6.85,0) {$\circ$};
\node[draw=paperviolet,thick,fill=violetbg,rounded corners=3pt,minimum width=3.6cm,minimum height=1.5cm,align=center] (cop) at (9.3,0)
  {Copula $C$ on $[0,1]^n$\\[2pt] the dependence};
\node[paperteal,font=\scriptsize\bfseries,anchor=north] at (marg.south) {what each variable does alone};
\node[paperviolet,font=\scriptsize\bfseries,anchor=north] at (cop.south) {how variables move together};
\node[align=center,font=\small,anchor=north] at (4.65,-1.35)
  {$F(x_1,\dots,x_n) = C\big(F_1(x_1),\,\dots,\,F_n(x_n)\big)$
   \ \textcolor{mutedgray}{--- unique when marginals are continuous}};
\end{tikzpicture}
\caption{Sklar's theorem (1959): every joint law factors into marginals plus
a dependence object on the unit cube.  Dependence is a separate,
exchangeable module: change $C$, keep every marginal fixed.  In density
form, $f(x_1,\ldots,x_n)=c(F_1(x_1),\ldots,F_n(x_n))\prod_i f_i(x_i)$, so
the copula density $c$ is a multiplicative correction to independence that
never touches the $f_i$.}
\label{fig:sklar}
\end{figure}

The construction proceeds in four stages. Section~\ref{sec:controller} forms
groups of related tokens and fixes a coupling strength before any routing
noise is drawn; Section~\ref{sec:copula} induces correlated Gumbel noise
within a group; Section~\ref{sec:preserved} establishes that every token
retains its original stochastic Top-$K$ law; and Section~\ref{sec:hierarchy}
introduces a tunably antithetic shared signal between paired groups to
counteract the load burstiness produced by positive within-group
coordination.

\subsection{A pre-routing coupling controller}
\label{sec:controller}
Partition a sequence into disjoint candidate groups $g$.  In the minimal
version, these are fixed windows of $m$ adjacent tokens.  A controller reads
only frozen, pre-routing features $s_g$, such as the mean hidden state, mean
gate entropy, within-group gate similarity, and boundary indicators:
\begin{equation}
  a_g=\sigma(\phi(s_g))\in[0,1],\qquad
  \rho_g=a_g\rho_{\max},\qquad 0\leq\rho_{\max}<1.
  \label{eq:controller}
\end{equation}
Here $\sigma$ is the logistic sigmoid, so that $a_g\in[0,1]$. The correlation
used by the sampler is $\rho_g$, capped below one by $\rho_{\max}$. Keeping
$\rho_{\max}<1$ leaves each token with private randomness and keeps the
Gaussian construction below nondegenerate.

\paragraph{A second dial between paired groups.}
Once groups are matched into disjoint pairs, the controller may also choose an
opposition strength
\begin{equation}
  \alpha_{g,g'}=\sigma\!\left(\phi_{\mathrm{pair}}
  (\stopgrad(s_g),\stopgrad(s_{g'}))\right)\in[0,1].
  \label{eq:pair-controller}
\end{equation}
The within-group strength $\rho_g$ and between-group strength
$\alpha_{g,g'}$ have different jobs.  $\rho_g$ makes tokens in the same group
share more randomness; $\alpha_{g,g'}$ determines how strongly the two groups'
shared random signals oppose each other.  A fixed value of either dial is also
valid; learning both is optional.

For a single expert coordinate, \CGA{} may assign the tokens of a group a
partly shared random perturbation toward that expert when their frozen router
scores are similar. This perturbation does not increase the expert's score,
alter the router weights, or force any token to select the expert; it affects
only the random tie-breaking component of otherwise unchanged routing
decisions. At $\rho_g=0$ these perturbations are independent, and as $\rho_g$
approaches one the shared component dominates, while each token retains the
same marginal noise distribution.

The controller must be evaluated before routing noise is drawn.  At $a_g=0$,
$\rho_g=0$ and the rule reduces exactly to independent stochastic Top-$K$.
Groups may be chosen adaptively from frozen inputs, and groups may further be
matched into disjoint pairs with tunable opposition (Section~\ref{sec:hierarchy}),
provided that group membership \emph{and} the pairing are fixed before any
routing noise is sampled and groups remain disjoint.

\paragraph{Why ``pre-routing'' is required.}
The group, its coupling strength, and its optional paired partner may
depend on information already available to the frozen model, such as hidden
states or gate similarity.  They must not depend on the Gumbel draws or on
the selected experts from the current routing operation.  Selecting a group
after seeing a favorable random outcome would bias the distribution and
destroy the preservation result.

\subsection{Copula-correlated Gumbel perturbations}\label{sec:copula}
For every expert $e$ and group $g$, independently draw a shared latent
$\zeta_{ge}\sim\Normal(0,1)$ and per-token noises
$\epsilon_{te}\sim\Normal(0,1)$:
\begin{equation}
  y_{te}=\sqrt{\rho_g}\,\zeta_{ge}
         +\sqrt{1-\rho_g}\,\epsilon_{te},\qquad
  u_{te}=\Phi(y_{te}),\qquad
  \gamma_{te}=-\log[-\log(u_{te})].
  \label{eq:copula-gumbel}
\end{equation}

\begin{figure}[t]
\centering
\begin{tikzpicture}[font=\small,
  stage/.style={rounded corners=3pt,minimum width=3.55cm,minimum height=1.35cm,align=center},
  arr/.style={-{Stealth[length=2mm]},ink,thick}]
\node[stage,draw=paperviolet,thick,fill=violetbg] (s1) at (0,0)
  {shared $\zeta_{ge}\sim\Normal(0,1)$\\[1pt] $+$ private $\epsilon_{te}\sim\Normal(0,1)$};
\node[stage,draw=ink] (s2) at (4.05,0)
  {$y_{te}=\sqrt{\rho_g}\,\zeta_{ge}$\\[1pt] $\phantom{y_{te}=}+\sqrt{1-\rho_g}\,\epsilon_{te}$};
\node[stage,draw=ink] (s3) at (8.1,0) {$u_{te}=\Phi(y_{te})$};
\node[stage,draw=paperteal,thick,fill=tealbg] (s4) at (12.15,0)
  {$\gamma_{te}=-\log(-\log u_{te})$};
\draw[arr] (s1) -- (s2); \draw[arr] (s2) -- (s3); \draw[arr] (s3) -- (s4);
\node[mutedgray,font=\scriptsize,anchor=north,align=center] at ($(s1.south)+(0,-0.08)$) {\emph{dependence injected here}};
\node[mutedgray,font=\scriptsize,anchor=north,align=center] at ($(s2.south)+(0,-0.08)$) {\emph{still $\Normal(0,1)$ marginally}};
\node[mutedgray,font=\scriptsize,anchor=north,align=center] at ($(s3.south)+(0,-0.08)$) {\emph{$\Uniform[0,1]$ marginally}};
\node[mutedgray,font=\scriptsize,anchor=north,align=center] at ($(s4.south)+(0,-0.08)$) {\emph{$\mathrm{Gumbel}(0,1)$ marginally}};
\node[ink,font=\small\bfseries,align=center] at (6.05,-1.5)
  {correlated across tokens $t$ in a group\quad$\cdot$\quad independent across
   experts $e$\quad$\cdot$\quad exactly Gumbel per coordinate};
\end{tikzpicture}
\caption{The copula-Gumbel pipeline of \eqref{eq:copula-gumbel}.  Dependence
enters once, through the shared latent $\zeta_{ge}$; every later map is a
strictly monotone marginal transform, so the dependence pattern survives
while each coordinate lands exactly on the $\mathrm{Gumbel}(0,1)$ marginal
that Gumbel-Top-$K$ routing requires.}
\label{fig:pipeline}
\end{figure}
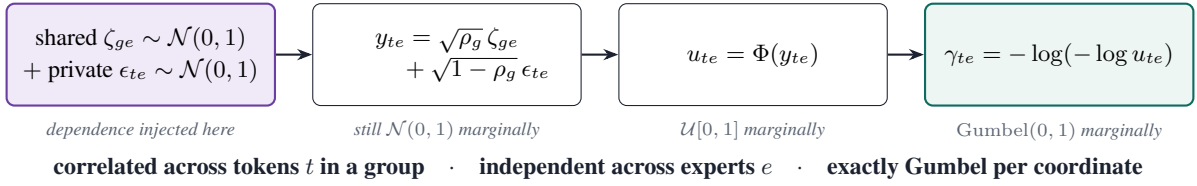

\paragraph{Distributional properties of the sampler.}
Figure~\ref{fig:pipeline} depicts the construction.
In \eqref{eq:copula-gumbel}, $y_{te}$ is a convex combination of the shared
latent $\zeta_{ge}$, common to every token in group $g$ at expert $e$, and the
private noise $\epsilon_{te}$, specific to token $t$. The coefficients
$\sqrt{\rho_g}$ and $\sqrt{1-\rho_g}$ are chosen so that $y_{te}$ is marginally
standard normal for every token, with $\rho_g$ setting the fraction of variance
attributable to the shared source: $\rho_g=0$ under independent routing, and
$\rho_g$ increases with coordination. The subsequent transformations alter the
marginal shape but not the dependence pattern: $u_{te}=\Phi(y_{te})$ is uniform
on $[0,1]$, and the inverse Gumbel c.d.f.\ maps $u_{te}$ to a standard Gumbel
variate $\gamma_{te}$, as required by Gumbel-Top-$K$ routing. This
Gaussian-to-uniform-to-Gumbel pipeline is the copula: it preserves each token's
marginal Gumbel law while inducing dependence across tokens.

The draws are independent across expert coordinates $e$.  This detail is
essential.  A token needs independent Gumbel perturbations across experts for
its usual ranked Top-$K$ distribution to remain valid.  We then apply the
unchanged Top-$K$ selection and mixture rule in
\eqref{eq:gumbel-topk} and \eqref{eq:mixture}.

Positive $\rho_g$ correlates the random perturbations received by the same
expert across tokens.  When related tokens have similar frozen logits, this
raises their chance of including the same experts.  It does not guarantee
same-expert inclusion for dissimilar gates.  And it has a quantifiable cost:
bunching a group's tokens onto shared experts makes their inclusion counts
positively correlated, so the \emph{variance} of realized per-expert loads
grows even though conditional expected loads at this layer are exactly
preserved.
Proposition~\ref{prop:variance} makes both statements precise: the
within-group and between-group dials jointly control coherence and load
dispersion, and
Section~\ref{sec:hierarchy} supplies the second pole.

\subsection{What is preserved}\label{sec:preserved}
\paragraph{The key distinction.}
The construction deliberately changes how tokens' random choices move
together.  It does not change the distribution of the random vector seen by
one token.  For example, two adjacent tokens may now select the same expert
more often, but if either token is considered alone and routed repeatedly, it
has exactly the same probabilities for every ordered Top-$K$ list as before.
The following theorem formalizes this statement.

\begin{theorem}[Per-token Top-$K$ routing-law invariance]\label{thm:law}
Condition on all frozen hidden states, router logits, group memberships, and
controller outputs.  Suppose that \CGA{} uses \eqref{eq:copula-gumbel} and that
the copula draws are independent across expert coordinates.  For every token
$t$,
\begin{equation}
  (\gamma_{t1},\ldots,\gamma_{tE})\overset{d}{=}
  (G_1,\ldots,G_E),\qquad G_e\overset{\mathrm{iid}}{\sim}
  \mathrm{Gumbel}(0,1).
  \label{eq:gumbel-law}
\end{equation}
Consequently, the ordered list $(r_{t1},\ldots,r_{tK})$, selected set $S_t$,
and weights $(w_{t1},\ldots,w_{tE})$ have exactly the same conditional
distribution as under the independent Gumbel-Top-$K$ router.
\end{theorem}
\begin{proof}
For every $t,e$, $y_{te}$ in \eqref{eq:copula-gumbel} is standard normal, so
$u_{te}$ is uniform and $\gamma_{te}$ is standard Gumbel.  Different coordinates
$e$ use independent $(\zeta_{ge},\epsilon_{te})$ draws, so
$(\gamma_{t1},\ldots,\gamma_{tE})$ is i.i.d.\ Gumbel.  The Top-$K$ map and the
gate-weight map in \eqref{eq:gumbel-topk}--\eqref{eq:mixture} are deterministic
functions of this vector and the fixed logits, which proves the result.
\end{proof}

\paragraph{What the theorem does and does not say.}
The theorem is conditional on the hidden states and logits entering this
routing layer.  In plain terms, once the frozen router has supplied the
scores for the current tokens, replacing independent noise by \CGA{} noise
does not change the probability distribution of any one token's route.  It
does change the joint distribution of several routes.  In a model with
multiple coupled MoE layers, that change can in turn affect later hidden
states and later logits.  Thus this is a precise layer-local safety guarantee,
not a claim that the entire multi-layer network has an unchanged output
distribution.

\begin{corollary}[Expected inclusion load is invariant]\label{cor:load}
Let $N_e=\sum_t\ind[e\in S_t]$ be the number of tokens that include expert
$e$.  Under the conditions of Theorem~\ref{thm:law},
\begin{equation}
  \E[N_e\mid\{\ell_t\}_t,\{a_g,\rho_g\}_g]
  = \sum_t\Prob\!\left(e\in S_t\mid\ell_t;\,
    \text{independent Gumbel-Top-}K\right).
  \label{eq:load}
\end{equation}
The right-hand side is the conditional expected inclusion count of the frozen
independent Gumbel-Top-$K$ router.  In particular, it does not depend on the
coupling strengths, groups, or fixed pairing once the layer's incoming logits
are fixed.
\end{corollary}
\begin{proof}
Apply Theorem~\ref{thm:law} to the indicator $\ind[e\in S_t]$ and sum over
tokens.  No independence across tokens is required for linearity of
expectation.
\end{proof}

\paragraph{Interpretation.}
The corollary follows from a simple accounting rule: expected load is the sum
of each token's expert-inclusion probability, and those individual
probabilities are unchanged.  It is stronger than preserving a softmax
coordinate.  Under Gumbel-Top-$K$, softmax values are not themselves
expert-inclusion probabilities.  \CGA{} preserves the complete base
distribution over the ordered expert list and thus every inclusion
probability implied by that distribution.

Expected load is not the same as realized load: realized per-batch counts
fluctuate around this conditional mean.
\CGA{} preserves the conditional average over many routing draws, but it does
\emph{not} preserve realized loads, load variance, capacity overflow, or a
deterministic top-$K$ model's output exactly.

\subsection{Hierarchical dependence control}\label{sec:hierarchy}
Positive coupling inside one group encourages local agreement, but that same
agreement can make the group's demand for an expert arrive in a burst.  The
second level of the hierarchy acts \emph{between} groups.  It provides an
opposing signal whose strength can be selected: if one group receives a shared
random push toward an expert, its matched partner can receive an independent,
partly opposite, or fully opposite shared push.

Match groups into disjoint pairs $(g,g')$, fixed before sampling.  In the
minimal version, consecutive non-overlapping windows can be paired.  More
generally, any fixed pre-routing matching rule is valid.
Figure~\ref{fig:two-level-hierarchy} illustrates the maximally opposed case,
$\alpha_{g,g'}=1$.  For each pair and each expert coordinate $e$, draw
independent $\zeta_{ge},\eta_{g'e}\sim\Normal(0,1)$ and set

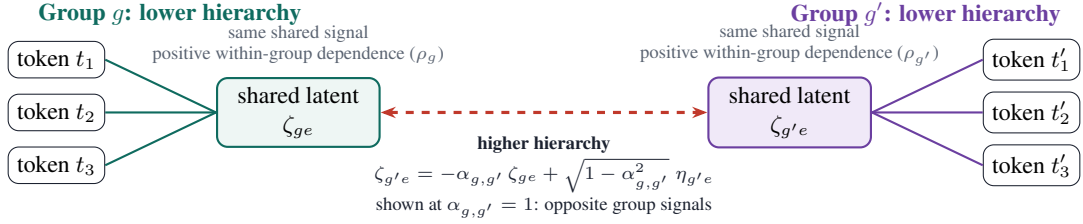
\begin{figure}[t]
\centering
\begin{tikzpicture}[
  token/.style={draw=ink,fill=white,rounded corners,minimum width=1.25cm,minimum height=0.45cm,
                align=center,font=\small},
  link/.style={thick},
  relation/.style={paperred,very thick,dashed,{Stealth[length=2mm]}-{Stealth[length=2mm]}}
]
  % Lower group.
  \node[font=\small\bfseries,paperteal] (labela) at (-1.7,1.65)
    {Group $g$: lower hierarchy};
  \node[token] (a1) at (-3.2,1.05) {token $t_1$};
  \node[token] (a2) at (-3.2,0.35) {token $t_2$};
  \node[token] (a3) at (-3.2,-0.35) {token $t_3$};
  \node[draw=paperteal,thick,fill=tealbg,rounded corners,minimum width=2.15cm,
        minimum height=0.75cm,align=center,font=\small] (za) at (0,0.35)
    {shared latent\\$\zeta_{ge}$};
  \draw[link,paperteal] (a1.east) -- (za.west);
  \draw[link,paperteal] (a2.east) -- (za.west);
  \draw[link,paperteal] (a3.east) -- (za.west);
  \node[align=center,font=\scriptsize,mutedgray] at (0,1.25)
    {same shared signal\\positive within-group dependence ($\rho_g$)};

  % Lower partner group.
  \node[draw=paperviolet,thick,fill=violetbg,rounded corners,minimum width=2.15cm,
        minimum height=0.75cm,align=center,font=\small] (zb) at (6.5,0.35)
    {shared latent\\$\zeta_{g'e}$};
  \node[token] (b1) at (9.7,1.05) {token $t'_1$};
  \node[token] (b2) at (9.7,0.35) {token $t'_2$};
  \node[token] (b3) at (9.7,-0.35) {token $t'_3$};
  \draw[link,paperviolet] (zb.east) -- (b1.west);
  \draw[link,paperviolet] (zb.east) -- (b2.west);
  \draw[link,paperviolet] (zb.east) -- (b3.west);
  \node[font=\small\bfseries,paperviolet] at (8.3,1.65)
    {Group $g'$: lower hierarchy};
  \node[align=center,font=\scriptsize,mutedgray] at (6.5,1.25)
    {same shared signal\\positive within-group dependence ($\rho_{g'}$)};
  \draw[relation] (za.east) -- node[below=5pt,align=center,font=\scriptsize,text=ink]
    {\textbf{higher hierarchy}\\[1pt]
     $\zeta_{g'e}=-\alpha_{g,g'}\,\zeta_{ge}+\sqrt{1-\alpha_{g,g'}^2}\;\eta_{g'e}$\\[1pt]
     shown at $\alpha_{g,g'}=1$: opposite group signals}
    (zb.west);
\end{tikzpicture}
\caption{The two-level dependence structure for one expert coordinate $e$.
At the lower level, all tokens in each group share that group's latent,
creating positive dependence \emph{within} Group $g$ and within Group $g'$.
At the higher level, $\alpha_{g,g'}$ controls the opposition between the two
group latents.  The diagram shows $\alpha_{g,g'}=1$, where paired groups
receive fully opposite shared random pushes.  At $\alpha_{g,g'}=0$, they are
independent.  Tokens inside either individual group remain positively
coordinated.  Private token noises $\epsilon_{te}$ are omitted for clarity.}
\label{fig:two-level-hierarchy}
\end{figure}

\noindent Figure~\ref{fig:two-level-hierarchy} shows the key point: opposition
belongs to the \emph{relationship between groups}, not to the relationship
among tokens inside Group $g'$.  Both groups have positive within-group
dependence; only the higher hierarchy controls how strongly their shared
signals oppose one another.

\begin{equation}
  \zeta_{g'e}
  =-\alpha_{g,g'}\zeta_{ge}
   +\sqrt{1-\alpha_{g,g'}^2}\,\eta_{g'e},
  \qquad 0\leq\alpha_{g,g'}\leq1,
  \label{eq:paired-latent}
\end{equation}
leaving the within-group construction \eqref{eq:copula-gumbel} and all
per-token noises $\epsilon_{te}$ unchanged and independent.  Whatever random
push the pair's first member gives its tokens toward expert $e$, the second
member receives an opposing shared component of strength $\alpha_{g,g'}$.
At $\alpha_{g,g'}=0$, the two group latents are independent; at
$\alpha_{g,g'}=1$, \eqref{eq:paired-latent} becomes the fully antithetic rule
$\zeta_{g'e}=-\zeta_{ge}$.  Pairing supports strong local negative dependence
without requiring a jointly negative equicorrelated vector across all groups,
whose feasible correlation is bounded below by $-1/(G-1)$ and vanishes as the
number of groups grows.

\paragraph{Why this does not break the one-token law.}
For any $\alpha_{g,g'}\in[0,1]$, the two terms on the right of
\eqref{eq:paired-latent} combine independent standard normals with squared
coefficients that sum to one.  The partner group's shared signal is therefore
still standard normal when viewed alone.  Each token continues to combine
that signal with private noise in exactly the way used in
\eqref{eq:copula-gumbel}.  The following corollary records the consequence.

\begin{corollary}[Hierarchical routing-law invariance]\label{cor:hier}
Under the paired-latent rule \eqref{eq:paired-latent}, for any fixed
$\alpha_{g,g'}\in[0,1]$, the conclusion of
Theorem~\ref{thm:law} and Corollary~\ref{cor:load} continue to hold for every
token.
\end{corollary}
\begin{proof}
Equation~\eqref{eq:paired-latent} is standard normal for every pair and
expert coordinate.  Pairs use independent draws across expert coordinates,
and each token belongs to exactly one group, hence touches exactly one shared
latent per expert coordinate.  Every $y_{te}$ is therefore still standard
normal with independent coordinates across $e$, and the proof of
Theorem~\ref{thm:law} applies verbatim.
\end{proof}

\begin{proposition}[Coherence--dispersion trade-off]\label{prop:variance}
Condition on all logits, group memberships, pairings, within-group coupling
strengths, and opposition strengths $\alpha_{g,g'}\in[0,1]$, and fix an
expert $e$.  Write $X_g=\sum_{t\in g}\ind[e\in S_t]$ and $N_e=\sum_g X_g$.
Then the conditional expected loads $\E[N_e]$ are identical under independent
routing, flat coupling (independent $\zeta_{ge}$ across groups), and the
tunable paired construction, and:
\begin{enumerate}
  \item[(i)] under flat coupling with any strengths $\rho_g\geq 0$,
  $\operatorname{Var}(N_e)$ is at least its value under independent routing;
  \item[(ii)] under paired coupling with the same $\rho_g$ and any
  $\alpha_{g,g'}\in[0,1]$, $\operatorname{Var}(N_e)$ is at most its value
  under flat coupling.
\end{enumerate}
\end{proposition}
Part (i) states that positive within-group coupling cannot make an expert's
conditional load less variable than under independent routing, which is the
cost of local coherence. Part (ii) states that adding cross-group opposition
to such coupling cannot make that load more variable than flat coupling at the
same strengths. The two constructions coincide at $\alpha=0$ and are fully
antithetic at $\alpha=1$; the opposition is a partial counterbalance, not a
guarantee of lower variance than independent routing.

\begin{proof}[Proof sketch]
Both parts follow from the association inequality for functions of
independent random variables \citep{esary1967association} after one sign
change of coordinates; the full argument is in the appendix.  For (i), given
the group's latent vector, token inclusion indicators within a group are
conditionally independent with conditional means that are coordinatewise
monotone in a common transformed latent, hence pairwise nonnegatively
correlated, so within-group count variance can only grow.  For (ii), after
averaging over a partner group's independent residual noise, its conditional
mean is coordinatewise nonincreasing in $W_g$.  The paired counts therefore
have nonpositive covariance; unpaired groups are independent, so total
variance can only shrink.
\end{proof}

Two qualifications delimit the claim.  First, the two bounds run in opposite
directions from different baselines: the hierarchical scheme reduces variance
relative to \emph{flat coupling at the same $\rho$}, not necessarily below
the independent-routing baseline; the dials interpolate, they do not
dominate.  Second, the magnitude of both effects depends on how strongly
inclusion probabilities respond to the shared latents, which varies with the
gate distribution; the proposition signs the effects but does not quantify
them, and heterogeneous gates can make either effect small.

The higher level also changes \emph{cross-group} joint statistics.  Increasing
$\alpha_{g,g'}$ strengthens the opposing shared component; the pilot's
synthetic check measures the resulting paired-boundary overlap directly.
Flat coupling ($\alpha_{g,g'}=0$) leaves their shared latents independent.
Applications that rely on cross-group co-occurrence patterns should treat
$\alpha_{g,g'}$ and the pairing as active design choices, not free lunches.

\subsection{Routing-only adaptation}\label{sec:adaptation}
This subsection explains one possible use of the dependence mechanism.  It
does not claim that routing-only adaptation is effective on every task.  The
fixed-coupling mechanism above needs no learning at all: a user can choose a
window size, a pairing rule, and values of $\rho_g$ and $\alpha_{g,g'}$.
Adaptation asks the
separate question of whether a small controller can choose the coupling
strength from frozen features while all ordinary model weights remain fixed.

Let $Y$ collect the correlated standard-normal variables $y_{te}$ in
\eqref{eq:copula-gumbel} across the groups in an input.  Conditional on the
frozen features and a fixed pairing, $Y$ has a Gaussian density
$p_\phi(Y\mid x)$ whose within-group correlations are set by
$\rho_g=\rho_{\max}\sigma(\phi(\stopgrad(s_g)))$ and whose paired
cross-group correlations are set by \eqref{eq:pair-controller}.  The
$\stopgrad(\cdot)$ operation guarantees that the controller's inputs provide
no gradient path into the backbone or router, so the routing-only property
would survive even if the base parameters were trainable.
With $\Theta_{\mathrm{base}}$ frozen, we optimize only $\phi$:
\begin{equation}
  \min_{\phi}\;
  \E_{(x,\tilde y),\,Y\sim p_{\phi}(Y\mid x)}
  \left[\mathcal{L}\big(F_{\Theta_{\mathrm{base}}}(x;Y),\tilde y\big)\right].
  \label{eq:peft-objective}
\end{equation}
Here $\tilde y$ denotes the target, avoiding a collision with the
pre-Gumbel normal variable $Y$.  The discrete Top-$K$ map prevents ordinary
pathwise differentiation through the selected set: a tiny change in a noise
value usually changes no selected expert, then abruptly changes the selected
set at a ranking boundary.  Instead, for
$L=\mathcal{L}(F_{\Theta_{\mathrm{base}}}(x;Y),\tilde y)$, a score-function
estimator \citep{williams1992simple} uses
\begin{equation}
  \nabla_\phi \E[L\mid x]
  =
  \E\!\left[
    \bigl(L-b(x)\bigr)
    \nabla_\phi\log p_\phi(Y\mid x)
    \,\middle|\,x
  \right],
  \label{eq:score-function}
\end{equation}
where $b(x)$ is any stop-gradient baseline that does not depend on the
realized routing noise, for example a moving-average or leave-one-out
baseline.  In the score term, $Y$ is treated as the sampled observation; the
gradient differentiates the closed-form Gaussian log-density with respect to
the controller parameters.  Thus the base model is evaluated only in the
forward direction, while the backward pass updates only $\phi$.  No gradient
through the discrete Top-$K$ operation is assumed or required.
Figure~\ref{fig:controller} summarizes the training loop.

\begin{figure}[t]
\centering
\begin{tikzpicture}[font=\small,
  inner/.style={draw=mutedgray,fill=white,rounded corners=3pt,minimum width=2.6cm,minimum height=1.2cm,align=center},
  arr/.style={-{Stealth[length=1.8mm]},ink,semithick}]
% frozen base
\node[inner] (hid) at (0,0) {hidden\\states $h_t$};
\node[inner] (log) at (3.2,0) {router logits\\$\ell_t$ (frozen)};
\node[inner] (loss) at (6.6,0) {Gumbel-Top-$K$\\$+$ experts $\to$ loss $L$};
\draw[arr] (hid) -- (log);
\draw[arr] (log) -- (loss);
\node[draw=mutedgray,dashed,rounded corners=5pt,fit=(hid)(loss),inner sep=10pt] (base) {};
\node[mutedgray,font=\scriptsize\bfseries,anchor=south] at (base.north)
  {frozen base model --- forward only, no backprop};
% controller
\node[draw=paperteal,very thick,fill=tealbg,rounded corners=4pt,align=center,minimum width=4.6cm,minimum height=1.5cm] (ctrl) at (11.6,0)
  {\textbf{trainable controller $\phi$}\\[2pt]
   $\rho_g=\rho_{\max}\,\sigma(\phi(s_g))$\\[1pt]
   $\alpha_{g,g'}=\sigma\big(\phi_{\mathrm{pair}}(s_g,s_{g'})\big)$};
% flows
\draw[-{Stealth[length=1.8mm]},paperteal,semithick]
  (hid.north) to[out=75,in=140] node[above,font=\scriptsize,paperteal,pos=0.45]
  {group features $s_g$ (stop-grad)} (ctrl.north west);
\draw[-{Stealth[length=1.8mm]},paperteal,semithick]
  (ctrl.west) -- node[above,font=\scriptsize,paperteal]
  {correlated noise $\gamma$} (loss.east);
% score-function gradient
\draw[-{Stealth[length=1.8mm]},paperred,thick,dashed]
  (loss.south) to[out=-70,in=-120] (ctrl.south);
\node[paperred,font=\scriptsize,align=center] at (7.6,-2.15)
  {$\nabla_{\!\phi}\,\E[L]=\E\big[(L-b)\,\nabla_{\!\phi}\log p_\phi(Y)\big]$
   --- score function, no backprop through Top-$K$};
\end{tikzpicture}
\caption{Routing-only adaptation.  The frozen base model is evaluated only
in the forward direction; the controller reads stop-gradient group features
$s_g$, sets the dependence dials $(\rho_g,\alpha_{g,g'})$ of the correlated
noise, and is trained by the score-function estimator
\eqref{eq:score-function}, whose gradient goes \emph{around} the discrete
Top-$K$ map rather than through it.  Trainable parameters range from one
scalar per MoE layer to a small MLP.}
\label{fig:controller}
\end{figure}
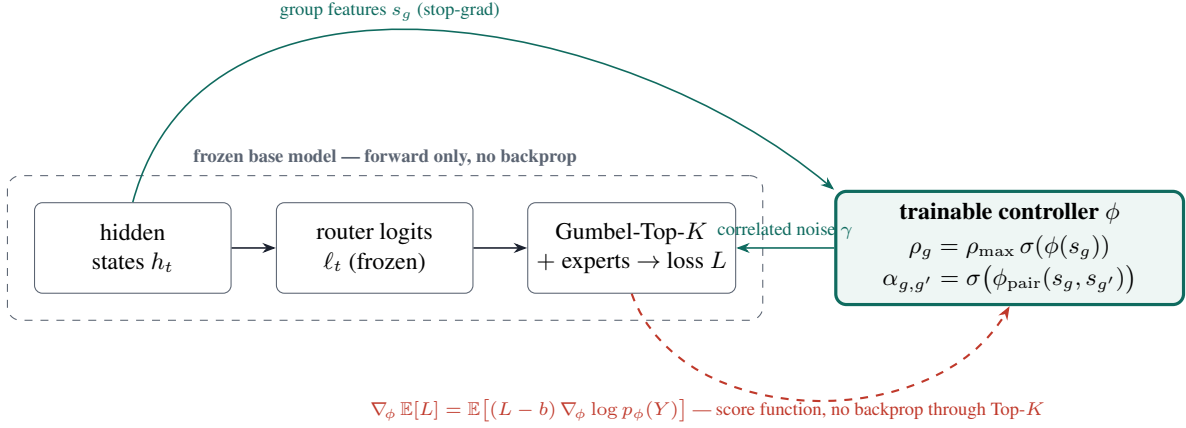

\paragraph{Score-function estimator.}
Rather than differentiating through the discrete selection, the estimator
reweights each sampled noise configuration by its centered loss and adjusts
the controller to raise or lower the probability of that configuration
accordingly. The term $\nabla_\phi\log p_\phi(Y\mid x)$ measures the
sensitivity of the sampled correlated noise's log-density to the controller's
coupling strength, and the baseline $b(x)$ subtracts a reference loss to reduce
the variance of the update. The frozen base model contributes only the forward
loss, and all trainable gradients terminate at the controller.

\paragraph{The density for the hierarchical sampler.}
For an unpaired group of $m$ tokens and one expert coordinate, the covariance
is $(1-\rho_g)I_m+\rho_g\mathbf{1}\mathbf{1}^{\mathsf T}$.  For an
equal-size paired group $(g,g')$, the corresponding $2m$-dimensional
covariance is
\begin{equation}
 \Sigma_{g,g'} =
 \begin{bmatrix}
 (1-\rho_g)I_m+\rho_g\mathbf{1}\mathbf{1}^{\mathsf T}
 &-\alpha_{g,g'}\sqrt{\rho_g\rho_{g'}}\,\mathbf{1}\mathbf{1}^{\mathsf T}\\
 -\alpha_{g,g'}\sqrt{\rho_g\rho_{g'}}\,\mathbf{1}\mathbf{1}^{\mathsf T}
 &(1-\rho_{g'})I_m+\rho_{g'}\mathbf{1}\mathbf{1}^{\mathsf T}
 \end{bmatrix}.
 \label{eq:paired-covariance}
\end{equation}
The full density $p_\phi$ is the product of these Gaussian blocks across
expert coordinates and disjoint groups or pairs.  It is nonsingular because
$\rho_{\max}<1$.  Thus \eqref{eq:score-function} applies to both the flat and
tunable paired variants when the relevant block log-density is used.  Pairing
itself remains fixed before sampling; learning a discrete pairing would
require a separate discrete gradient estimator.

Readers need not use this covariance expression to understand or run the
fixed-coupling router.  It is included to make the learning claim precise:
when paired groups are present, their noises must be scored jointly rather
than as two independent groups.  At $\alpha_{g,g'}=0$ the off-diagonal blocks
vanish and the density reduces to independent group blocks; at
$\alpha_{g,g'}=1$ it is the fully antithetic density.

\paragraph{The baseline used in the pilot.}
Equation~\eqref{eq:score-function} is exactly unbiased for a baseline that
does not depend on the scored sample's routing noise.  The small pilot uses
the batch-mean loss as a practical baseline.  Because that mean includes the
scored sample, its expectation is a $(1-1/B)$-scaled version of the
score-function gradient for batch size $B$; a leave-one-out baseline would
remove this finite-batch scaling.  This choice affects optimization
efficiency, not the routing-law theorem.

A biased continuous Top-$K$ relaxation may instead be used during
optimization, while evaluating the exact discrete rule.  The routing-law
results concern the exact execution distribution, not a training relaxation.
The controller may set both within-group strengths $\rho_g$ and paired-group
opposition strengths $\alpha_{g,g'}$.

\paragraph{Where a learning signal can arise.}
Marginal preservation means that an objective depending only on one token's
isolated routing distribution cannot identify $\phi$: that distribution does
not change.  Any useful signal must arise from joint effects, for example
several routed positions interacting through later attention, multiple MoE
layers compounding, or an explicitly joint routing objective.  Consequently,
ordinary per-token cross-entropy can provide a weak or noisy signal in a
fully frozen model.  This is a prediction of the construction, not a failure
of the estimator, and motivates the cautious pilot in
Section~\ref{sec:pilot}.

%==============================================================================
\section{Related Work}\label{sec:related}
%==============================================================================
\paragraph{MoE routing.}
Sparse MoE systems commonly use top-$K$ selection or capacity-aware assignment
\citep{shazeer2017outrageously,lepikhin2021gshard,fedus2022switch,
zhou2022mixture,lewis2021base}.  SeqTopK redistributes a sequence-level expert
budget \citep{wen2025seqtopk}; it changes the number of experts assigned to
individual tokens.  Similarity-aware routers modify gates or selections using
token relationships \citep{nguyen2025graphtokens,omi2025simbal}.  \CGA{}
addresses a different degree of freedom: it holds each token's stochastic
Top-$K$ routing law fixed and modifies only joint assignments across tokens,
in either direction.

\paragraph{Routing-aware PEFT.}
LoRA \citep{hu2022lora} and related PEFT methods adapt weights while freezing
most of a backbone.  Routed PEFT learns or reuses routing to activate adapter
modules \citep{liu2026perft}.  These approaches allocate trainable
\emph{representational} capacity.  The controller of
Section~\ref{sec:adaptation} instead adds no expert-weight adapter: it is a
small trainable map into the dependence structure of a frozen base router's
sampling noise.

\paragraph{Dependent discrete sampling and variance reduction.}
The Gumbel-Top-$K$ trick samples a ranked set without replacement
\citep{kool2019gumbel}; \citet{huijben2022gumbel} survey structured
extensions of the Gumbel-max mechanism.  Copulas separate joint dependence
from marginals \citep{sklar1959fonctions,nelsen2006copulas}.  Antithetic
variates are a classical variance-reduction device
\citep{hammersley1956antithetic}, and antithetic constructions also underlie
low-variance gradient estimators for discrete variables \citep{yin2019arm};
the association inequality of \citet{esary1967association} supplies the
monotonicity argument we use.  Our construction combines these elements at
the routing-noise level of an MoE: coordinate-wise copulas induce cross-token
dependence of either sign while retaining the i.i.d.\ Gumbel vector required
by each token's original Gumbel-Top-$K$ law.

%==============================================================================
\section{Initial Frozen-Base Pilot}\label{sec:pilot}
%==============================================================================
This pilot is deliberately small.  Its purpose is to test whether the exact
sampler, the invariance checks, and the controller-only training route work
together.  It is \emph{not} evidence that routing-only adaptation improves a
pretrained MoE on a downstream task.

\paragraph{Setup.}
We trained one 15.8M-parameter, six-layer decoder-only Top-2 MoE language
model on 10M TinyStories tokens.  It has eight experts in each of three MoE
layers and processes sequences of length 256.  Groups are fixed,
non-overlapping windows of $m=4$ adjacent tokens; within-window Jaccard means
the average Top-2 set Jaccard similarity of adjacent positions inside those
windows.  We then froze every base parameter and evaluated the copula
mechanism on the held-out validation split.  The main fixed-coupling
comparison uses $\rho=0.6$ and $\alpha=0$, so the reported full-model rows
measure lower-level positive coupling only.

The learned-controller experiment is intentionally a restricted special case
of Section~\ref{sec:adaptation}: it learns one constant $\rho$ per MoE layer
(three parameters total), rather than a feature-conditioned, per-group map
$\phi(s_g)$.  It trains for 2M additional in-distribution tokens using the
score-function route in \eqref{eq:score-function}, $\rho_{\max}=0.95$, and
three seeds.  The learned-controller and table rows use $\alpha=0$; the
higher-level dial is evaluated separately below with fixed coupling only, so
no run in this pilot measures a learned $\alpha$ or an end-to-end
variance-reduction benefit.

\paragraph{Routing-law checks.}
In a separate synthetic fixed-logit test with four tokens, 60,000 draws, and
$\rho=0.8$, the largest difference between the empirical frequency of an
ordered Top-2 list under independent and copula routing was $0.0035$; the
largest difference in an expert-inclusion frequency was
$2.2\times10^{-4}$.  These are Monte-Carlo checks, not proofs;
Theorem~\ref{thm:law} gives the exact conditional result.  In this deliberately
similar-logit synthetic setting, adjacent-token Top-2 Jaccard overlap rose
from $0.53$ to $0.77$.  These values are not directly comparable with the
full-model metrics in Table~\ref{tab:pilot}.

We also test the new higher-level dial on identical synthetic logits, which
isolate the shared-noise effect from gate differences.  At
$\rho=0.6$, paired-boundary Jaccard overlap decreases from $0.389$ at
$\alpha=0$ to $0.303$ at $\alpha=0.5$ and $0.218$ at $\alpha=1$, while
within-window overlap remains approximately constant ($0.601$, $0.601$, and
$0.600$).  A separate $\alpha=1$ law check gives maximum ordered-list and
inclusion-frequency deviations of $0.0028$ and $0.0034$, respectively.  These
synthetic checks validate that $\alpha$ changes the intended cross-group
statistic without altering the observed per-token law; they are not
end-to-end evidence of load balancing.

The same pattern holds end to end on the frozen model.  A single-seed
evaluation-only sweep over $\rho\in\{0.6,0.9\}$ and
$\alpha\in\{0,0.5,1\}$ lowers paired-window-boundary Jaccard from $0.210$ to
$0.170$ to $0.137$ at $\rho=0.6$, and from $0.210$ to $0.152$ to $0.110$ at
$\rho=0.9$, while within-window Jaccard is unchanged to within $3\times
10^{-4}$ ($0.314$ and $0.389$, respectively) and validation cross-entropy
varies by at most $3.1\times10^{-4}$ across all six cells.  The $\alpha$ dial
therefore moves only its target statistic in a full multi-layer model as
well.  Aggregate load CV again changes very little, and, as above, that
summary does not test Proposition~\ref{prop:variance}.

\begin{table}[t]
\centering
\caption{Initial frozen-base pilot.  Values after $\pm$ are standard
deviations over three seeds; the two fixed rows are one
mechanism-check run each.  The pilot is not a task-adaptation benchmark.}
\label{tab:pilot}
\small
\begin{tabular}{lrrrrr}
\toprule
Method & Trainable & Validation & Within-window & Distinct experts & Observed load \\
& parameters & CE & Jaccard & per window & CV \\
\midrule
Independent ($\rho=0$) & 0 & 3.12832 & 0.209 & 5.277 & 0.15946 \\
Fixed copula ($\rho=0.6$) & 0 & 3.12816 & 0.314 & 4.605 & 0.15950 \\
Learned scalar controller & 3 & $3.12837\pm0.00003$ &
  $0.228\pm0.004$ & $5.132\pm0.027$ & $0.15907\pm0.00012$ \\
Router-LoRA reference & 1,584 & $3.12439\pm0.00034$ &
  $0.208\pm0.001$ & $5.282\pm0.005$ & $0.15055\pm0.00457$ \\
\bottomrule
\end{tabular}
\end{table}

\paragraph{Findings and limits.}
Fixed positive coupling substantially changes the \emph{joint} routing
statistics: it raises within-window overlap and reduces the number of
distinct experts used in a window.  The observed validation cross-entropy
difference is small in this one-seed check, but layer-local routing-law
invariance does not predict sequence-level cross-entropy invariance.  This is
mechanism evidence, not a quality-improvement claim.  The observed aggregate
load CV also changes very little.  It is a finite-sample summary across
experts, rather than a direct estimate of the conditional variance in
Proposition~\ref{prop:variance}, so it does not test that proposition's
signed comparison.

The three-scalar controller trains stably and preserves aggregate inclusion
frequencies to within $2.4\times10^{-4}$ of the independent estimate in this
evaluation.  This is an empirical observation, not an end-to-end invariance
guarantee.  Its learned strengths range from $0.04$ to $0.27$ without a
reproducible pattern, and it provides no validation cross-entropy gain.  The
much smaller Jaccard change than the fixed-$\rho$ row is consistent with the
learned strengths staying well below $0.6$ and with the weak joint signal
discussed above.  Router-LoRA attains lower cross-entropy, but it has 528
times as many trainable parameters, optimizes the base balance regularizer in
addition to cross-entropy, and changes the routing distribution itself: its
maximum observed inclusion-frequency shift is $5.8\times10^{-3}$.  It is
therefore an illustrative conventional routing-adaptation reference, not a
matched-budget or matched-objective control.

The pilot leaves the central empirical questions open: whether learned
dependence helps on a real domain shift, whether an $\alpha$ sweep reduces
capacity overflows in practice, and whether fewer distinct experts translate
into measurable hardware savings.  A paper-level evaluation should answer
these questions on a pretrained stochastic Top-$K$ MoE with capacity-aware
measurements.

%==============================================================================
\section{Discussion and Limitations}\label{sec:discussion}
%==============================================================================
\CGA{} is a dependence layer for stochastic routers, not a universal drop-in
replacement for every MoE.  It is exactly plug-compatible only with a
stochastic Gumbel-Top-$K$ base router; converting a deterministic pretrained
Top-$K$ router to stochastic Gumbel-Top-$K$ changes its base behavior even at
zero coupling.  Capacity clipping, token dropping, and expert-choice
allocation act after sampling and are outside the invariance results; under a
hard capacity, the burstiness signed by Proposition~\ref{prop:variance}(i) is
exactly the quantity that causes overflow, which is one motivation for the
higher-level opposition dial.  The proposition signs but does not quantify either
variance effect, and for heterogeneous gate distributions the magnitudes may
be small.  The initial pilot in Section~\ref{sec:pilot} checks the
positive-coupling mechanism and the $\alpha$ dial's routing statistics, both
synthetically and end to end on the frozen model, but it does not measure
conditional load variance, capacity overflows, or a systems benefit.  Finally, the method requires $E$
Gumbel perturbations per token;
although routers usually already score all experts, the sampling and
cross-token coordination overhead must be measured.  Marginal preservation is
a safety and identifiability property, not an accuracy theorem: it does not
by itself imply better task performance, communication, specialization, or
realized load balance.

In summary, holding every token's Top-$K$ routing law fixed leaves a usable
design space---the joint dependence of routing across tokens. A hierarchical
copula reaches both poles of this space, coherence within groups and
dispersion across them, with provable invariance of every per-token routing
quantity and a signed characterization of the load-variance trade-off, and it
exposes these controls to a small controller trainable without updating the
frozen base. These guarantees are conditional and routing-layer-local, and do
not by themselves render a multi-layer MoE end-to-end invariant.

%==============================================================================
\section*{Appendix: Proof Details}
\paragraph{Independence across expert coordinates.}
Independence across expert coordinates is essential in Theorem~\ref{thm:law}.
Correlating $\gamma_{te}$ and $\gamma_{te'}$ within a token would alter the
Gumbel-Top-$K$ ranking law and invalidate the theorem.  Similarly, choosing
group membership, the pairing, $\rho_g$, or $\alpha_{g,g'}$ after observing
routing noise can select on the random variables and need not preserve their
uniform/Gumbel margins; the controller and pairing must be measurable with
respect to frozen, pre-noise quantities only.  The paired-latent rule
\eqref{eq:paired-latent} must likewise act coordinate-wise: coupling
$\zeta_{g'e}$ to $\zeta_{ge'}$ for $e\neq e'$ would couple expert coordinates
within a token and break the theorem.

\paragraph{Proof of Proposition~\ref{prop:variance}.}
Fix expert $e$ and condition throughout on logits, groups, pairings, and
strengths.  For a group $g$, collect its latents into
$\zeta_g=(\zeta_{g1},\ldots,\zeta_{gE})$ and define the transformed vector
$W_g$ by $W_{ge}=\zeta_{ge}$ and $W_{ge'}=-\zeta_{ge'}$ for $e'\neq e$.  The
coordinates of $W_g$ are independent symmetric standard normals.  For a token
$t\in g$, the indicator $\ind[e\in S_t]$ is, for fixed values of all other
noise, nondecreasing in the perturbed logit $\ell_{te}+\gamma_{te}$ and
nonincreasing in each competitor's perturbed logit; since $\gamma_{te}$ is
increasing in $y_{te}$, which is increasing in $\zeta_{ge}$, and each
competitor coordinate is increasing in $\zeta_{ge'}$, the conditional mean
$\varphi_t(W_g)=\Prob(e\in S_t\mid W_g)$ (integrating out the token noises
$\epsilon$) is nondecreasing in every coordinate of $W_g$.  Given $W_g$, the
indicators of distinct tokens in $g$ are conditionally independent, because
they involve disjoint sets of $\epsilon$ draws.

(i) For $t\neq t'$ in $g$,
$\operatorname{Cov}(\ind[e\in S_t],\ind[e\in S_{t'}])
=\operatorname{Cov}(\varphi_t(W_g),\varphi_{t'}(W_g))\geq 0$, since
independent coordinates are associated and both functions are coordinatewise
nondecreasing \citep{esary1967association}.  Hence
$\operatorname{Var}(X_g)\geq\sum_{t\in g}\operatorname{Var}(\ind[e\in S_t])$,
and the individual variances equal their independent-routing values by
Theorem~\ref{thm:law}.  Groups are independent under flat coupling, so
summing over $g$ proves (i).

(ii) Under the paired rule \eqref{eq:paired-latent}, let $V_{g'}$ be the
transformed independent standard-normal vector formed from $\eta_{g'e}$.
\[
  W_{g'}=-\alpha_{g,g'}W_g
  +\sqrt{1-\alpha_{g,g'}^2}\,V_{g'}.
\]
Given $(W_g,V_{g'})$, the counts $X_g$ and $X_{g'}$ are conditionally
independent because their $\epsilon$ draws are disjoint.  Their conditional
means are $\Phi_g(W_g)=\sum_{t\in g}\varphi_t(W_g)$, which is
coordinatewise nondecreasing, and
$\Phi_{g'}(-\alpha_{g,g'}W_g+\sqrt{1-\alpha_{g,g'}^2}V_{g'})$.  Averaging the
latter over $V_{g'}$ gives a function of $W_g$ that is coordinatewise
nonincreasing for $\alpha_{g,g'}\geq0$.  Therefore
\[
\begin{aligned}
  \operatorname{Cov}(X_g,X_{g'})
  &=\operatorname{Cov}\!\left(
    \Phi_g(W_g),
    \E_{V_{g'}}[
      \Phi_{g'}(-\alpha_{g,g'}W_g
      +\sqrt{1-\alpha_{g,g'}^2}V_{g'})
      \mid W_g]\right)\\
  &\leq 0
\end{aligned}
\]
by the same association inequality applied to a coordinatewise
nondecreasing and a coordinatewise nonincreasing function.  The marginal law
of each group is unchanged by the pairing, distinct pairs are independent,
and unpaired covariances vanish, so
$\operatorname{Var}(N_e)=\sum_g\operatorname{Var}(X_g)
+2\sum_{(g,g')}\operatorname{Cov}(X_g,X_{g'})$ is at most its flat-coupling
value.  Expected loads agree in all schemes by
Corollaries~\ref{cor:load} and~\ref{cor:hier}, conditional on the fixed logits
in this proof. \qed

\bibliographystyle{plainnat}
\bibliography{references}
\end{document}